\documentclass[journal]{IEEEtran}

\usepackage{setspace}

\usepackage{amsmath,amssymb,bm}

\usepackage{amsfonts,dsfont,color,bbm}
\usepackage[ruled,vlined]{algorithm2e}

\usepackage{booktabs}
\usepackage{multirow}
\usepackage{cite}

\usepackage{amsthm}
\usepackage{xpatch}
\makeatletter
\xpatchcmd{\proof}{\@addpunct{.}}{\@addpunct{:}}{}{}
\makeatother

\usepackage[bookmarks=false]{hyperref}
\usepackage{graphicx}
\usepackage{balance}

\renewcommand{\S}{\mathcal{S}}
\renewcommand{\L}{\mathcal{L}}
\newcommand{\muu}{\vec{\mu}}
\newcommand{\R}{\mathcal{R}}
\newcommand{\F}{\mathcal{F}}
\newcommand{\D}{\mathcal{D}}
\newcommand{\Q}{\mathcal{Q}}
\newcommand{\B}{\mathcal{B}}
\newcommand{\W}{\vec{W}}
\renewcommand{\b}{\vec{b}}
\renewcommand{\vec}[1]{\mbox{\boldmath${#1}$}}
\newcommand{\norm}[1]{\left|\left|#1\right|\right|}

\renewcommand{\and}{ \,\mathrm{and}\, }

\newcommand{\real}{\mathbb{R}}

\newcommand{\I}{\mathcal{I}}
\newcommand{\V}{\mathcal{V}}
\renewcommand{\P}{\mathcal{P}}
\renewcommand{\R}{\mathcal{R}}

\newcommand{\M}{\mathcal{M}}
\newcommand{\card}[1]{\vert{#1}\vert}

\newcommand{\Lsupervised}{\L_\mathrm{supervised}}
\newcommand{\Lsimilarity}{\L_\mathrm{instance}}
\newcommand{\Lcenter}{\L_\mathrm{center}}

\newcommand{\x}{\vec{x}}
\newcommand{\MC}{\card{\M}}

\newcommand{\fbest}{f_\mathrm{best}}
\newcommand{\Sdis}{\mathcal{S}_{\mathrm{dissimilarity}}}
\newcommand{\Scenter}{\mathcal{S}_{\mathrm{center}}}

\DeclareMathOperator*{\E}{\mathbb{E}}

\theoremstyle{plain}

\newtheorem{lemma}{Lemma}

\newtheoremstyle{customremark}
{}                
{}                
{}        
{}                
{\itshape\bfseries}       
{.}               
{ }               
{}                

\theoremstyle{customremark}
\newtheorem{remark}{Remark}
\theoremstyle{definition}
\newtheorem{definition}{Definition}

\newcommand{\be}{\begin{equation}}
\newcommand{\ee}{\end{equation}}
\newcommand{\bea}{\begin{eqnarray}}
\newcommand{\eea}{\end{eqnarray}}

\newcommand{\ei}{\end{itemize}}
\newcommand{\bi}{\begin{itemize}}

\begin{document}
%
\title{Unsupervised Anomaly Detection via \\Deep Metric Learning with End-to-End Optimization}

\author{Selim~F.~Yilmaz~and~Suleyman~S.~Kozat,~\IEEEmembership{Senior Member,~IEEE}
\thanks{This work is supported by the Turkish Academy of Sciences Outstanding Researcher Programme and also supported in part by Tubitak Contract No: 117E153.}
\thanks{S.~F.~Yilmaz~and~S.~S.~Kozat are with the Department of Electrical and Electronics
Engineering, Bilkent University, 06800 Ankara, Turkey (e-mail:
\{syilmaz,kozat\}@ee.bilkent.edu.tr).
}
\thanks{
S.~S.~Kozat is also with  the DataBoss A.S., Bilkent Cyberpark, Ankara 06800, (email: serdar.kozat@data-boss.com.tr).
}
}


\maketitle

\begin{abstract}
We investigate unsupervised anomaly detection for high-dimensional data and introduce a deep metric learning (DML) based framework. In particular, we learn a distance metric through a deep neural network. Through this metric, we project the data into the metric space that better separates the anomalies from the normal data and reduces the effect of the curse of dimensionality for high-dimensional data. We present a novel data distillation method through self-supervision to remedy the conventional practice of assuming all data as normal. We also employ the hard mining technique from the DML literature. We show these components improve the performance of our model and significantly reduce the running time. Through an extensive set of experiments on the 14 real-world datasets, our method demonstrates significant performance gains compared to the state-of-the-art unsupervised anomaly detection methods, e.g., an absolute improvement between 4.44\% and 11.74\% on the average over the 14 datasets. Furthermore, we share the source code of our method on Github to facilitate further research.
\end{abstract}

\IEEEpeerreviewmaketitle

\section{Introduction}
\subsection{Preliminaries}
\IEEEPARstart{W}{e} study anomaly detection~\cite{chandola2009anomaly}, which has been extensively studied over recent years due to its wide range of applications such as video surveillance~\cite{wu2019deep}, price manipulation detection~\cite{cao2014adaptive}, network intrusion detection~\cite{ahmed2016survey} and data imputation~\cite{ozkan2015imputation}. Anomaly detection aims to find the patterns that do not conform to the expected behavior~\cite{chandola2009anomaly}. Anomaly detection is an essential problem due to its usefulness in providing critical and actionable information on various domains~\cite{chandola2009anomaly}. Particularly, we study the anomaly detection problem in the unsupervised setting, where no labels exist. In the anomaly detection problem, the "normal" data is generally assumed to conform to a model, which is often unknown. Then, the anomalies are defined as the instances that significantly deviate from the underlying model of the normal data, where "significant deviances" are hard to model in real life and usually application dependent~\cite{aggarwal2015outlier}. In this sense, we introduce an algorithm to remedy such application dependent components and introduce a generic and widely applicable algorithm in different domains.


Accurately labeling anomalies under human supervision is costly and even impractical in many cases such as manually labeling anomalies in large social networks~\cite{aggarwal2015outlier}. Moreover, anomalies generally are rare among the normal behavior of the data, thus, filtering the normal data out requires a vast amount of human effort~\cite{aggarwal2015outlier}. Furthermore, new types of anomalies may arise, e.g., zero-day attacks~\cite{ahmed2016survey}, which do not exist in the set of the labeled anomalies~\cite{chandola2009anomaly}. Accordingly, an unsupervised anomaly detection method is more broadly applicable compared to the semi-supervised and the supervised ones, however, more prone to increased false-alarm rates, which severally limits their direct application in real life problems~\cite{chandola2009anomaly}.

Although significant development has been made in recent years~\cite{aggarwal2015outlier}, anomaly detection in unsupervised setting remains challenging, especially on higher dimensions. Moreover, the running time naturally increases as the number of the dimensions increase, which should be avoided since anomalous behavior should be determined in real time for quick action. To reduce the effect of the curse of dimensionality, various two-step approaches have been proposed~\cite{aggarwal2015outlier}. In these methods, first, the data is converted into low dimensional latent vectors via random subspace selection, dimensionality reduction, or random projection~\cite{aggarwal2015outlier}. Then, in the second step, a model is fitted to represent the normality in the lower-dimensional data, which is later used to score the anomalies. However, the initial dimensionality reduction step is independent of the latter model fitting step, which produces suboptimal performance since the essential information to detect the anomalies could be removed in the first stage.

To remedy those issues, we introduce a novel "end-to-end" anomaly detection method based on deep metric learning (DML) for the unsupervised setting. In particular, we project the data through a network into the metric space, in which we optimize the similarity-based loss between distilled training instances, where both components are optimized jointly, unlike the previous approaches \cite{aggarwal2015outlier}. We derive an anomaly scoring function with $\Theta(1)$\footnote{$\Theta(w(n))$ denotes the set of all $r(n)$, where $a_1w(n) \leq r(n) \leq a_2w(n)$, $\forall n > n_0$ for $n\in \mathbb{Z}^+$ such that there exist positive integers $a_1$, $a_2$, and $n_0$.} time and memory complexity. We introduce a self-supervised data distillation method, which naturally fits our derived scoring function. Specifically, throughout the epochs, we select the instances that are labeled as likely to be normal by our method and optimize the model only using these instances. We also introduce hard normal mining method for our unsupervised framework~\cite{harwood2017smart}. Through our hard normal mining method, we only learn from the instances that are the farthest to the metric-space center of our distilled data, i.e., the most challenging instances, which further improves the model by avoiding overfitting to the easiest instances. We precisely define our methodology so that further instance mining methods and loss functions can be readily adapted from the DML literature using our method. Moreover, our method can be readily adapted to other data structures, e.g., convolutional neural networks (CNNs)~\cite{krizhevsky2012imagenet} can be used for anomaly detection on images, and recurrent neural networks (RNNs)~\cite{elman1990finding} can be used for anomaly detection on time series as we provide such adaptions as remarks in the paper.
\subsection{Prior Art and Comparisons}
Deep neural networks have been extensively studied in different areas, including anomaly detection due to their ability to model complex patterns~\cite{lecun2015deep}.  To benefit from this capability, metric learning is extended to the deep learning framework in many forms such as the Siamese network~\cite{bertinetto2016fully}. The Contrastive Loss~\cite{hadsell2006dimensionality} is also proposed, which pulls the instances of the same class and push the instances of different classes in the metric space. This loss function later extended in several studies and achieved outstanding results in variety of tasks such as biometric recognition~\cite{biometrictnnls}, image classification~\cite{imageclassificationtnnls}, and image verification~\cite{faraki2017large}. Promoting hard pairs in training, i.e., hard mining, improves the discrimination capability of the DML model in the supervised setting~\cite{harwood2017smart}. Accordingly, we introduce hard normal mining method into our model for the unsupervised anomaly detection problem. Although the DML based methods received great attention in such supervised tasks, they have not received enough attention in the unsupervised anomaly detection literature. We are the first in the literature to introduce an unsupervised anomaly detection method based on the DML, which is optimized end-to-end.

In the classical anomaly detection literature, various methods have also been proposed that incorporate support vector machines (SVM)~\cite{scholkopf2001estimating}, density estimation~\cite{levelset}, information theory~\cite{livi2015entropic}, clustering~\cite{moshtaghi2011clustering}, and nearest-neighbor graphs~\cite{ramaswamy2000efficient}. Some hybrid approaches~\cite{erfani2016high} employ deep networks such as CNN to extract features and feed these features into classical anomaly detectors. Deep learning-based unsupervised anomaly detection methods almost exclusively aim to minimize the reconstruction error of autoencoder~\cite{aggarwal2015outlier,deeooc}, or the One-Class SVM (OC-SVM) formulation~\cite{Ergen_2019,wu2019deep}. Autoencoder (AE) based approaches require the encoder and the decoder networks. The encoder network projects the input into latent space. The decoder network tries to reconstruct the original input from the latent space. The decoder and the encoder structure is often symmetric. Our model requires only the encoder network and directly optimizes it end-to-end. Hence, our model requires nearly half of the parameters of the AE to project into the latent (metric) space, which is advantageous since the training accelerates as the number of parameters that need to be learned reduces~\cite{aggarwal2015outlier}. Moreover, we distill the data from possible outliers using our model, whereas AE treats outliers as normal instances, which causes overfitting since neural networks are sensitive to the outliers~\cite{aggarwal2015outlier}.

In this work, we use the DML in unsupervised anomaly detection domain. The common practice in the unsupervised anomaly detection is to assume all the training data as normal~\cite{harwood2017smart}. However, this assumption does not always hold in practice since the outliers occur within the typical process in many systems. Thus, such methods suffer in performance when the training data is polluted with the anomalies~\cite{chandola2009anomaly} (see Section~\ref{sec:unsupervised_experiments}). To reduce the effect of the polluted data, we introduce a data distillation method, which cleans the training data from outliers before each epoch through self-supervision. We also apply hard normal mining to prevent overfitting, which is a common issue in neural network based anomaly detection methods~\cite{aggarwal2015outlier}.
\subsection{Contributions}
Our contributions are as follows.
\begin{enumerate}
    \item For the first time in the literature, we introduce a DML based unsupervised anomaly detection method, where the method is optimized "end-to-end" in a fully unsupervised setting. 
    \item We derive a center distance based anomaly scoring function with $\Theta(1)$ time and memory complexity compared to the dissimilarity based scoring function of $\Theta(N)$ time and memory complexity where $N$ is the number of instances in the training set. 
    \item We introduce a novel self-supervision based data distillation method, which naturally fits into our prediction function with constant time complexity. This data distillation methodology is generic so that one can quickly adapt to any epoch-wise method such as deep autoencoders~\cite{aggarwal2015outlier}.
    \item We also present an online hard normal mining algorithm that selects the most challenging instances in the distilled training data, which prevents the overfitting to the easiest instances. 
    
    \item Through an extensive set of experiments on 14 distinct real-world datasets, we demonstrate that our method significantly outperforms the state-of-the-art methods for unsupervised and one-class training settings. Our method achieves an absolute improvement between 4.44\% and 11.74\% on the average over 14 datasets that are widely used in the anomaly detection literature.

    \item We perform an ablation study to demonstrate the performance gains achieved by the components of our method. 
    
    \item For reproducibility and to facilitate further research, the source code of our method is available on \href{http://github.com/selimfirat/addml/}{http://github.com/selimfirat/addml/}.
    

\end{enumerate}
\subsection{Organization of the Article}
The organization of the rest of this article is as follows. In Section~\ref{sec:problem}, we first describe the unsupervised anomaly detection problem and define our objective. In Section~\ref{sec:novel_anomaly}, we develop our method. In Section~\ref{sec:experiments}, we describe the performed experiments. Finally in Section~\ref{sec:conclusion}, we conclude by providing remarks.
\section{Problem Description}
\label{sec:problem}
\begin{figure*}[t]

    \centering

    \includegraphics[width=0.8\textwidth]{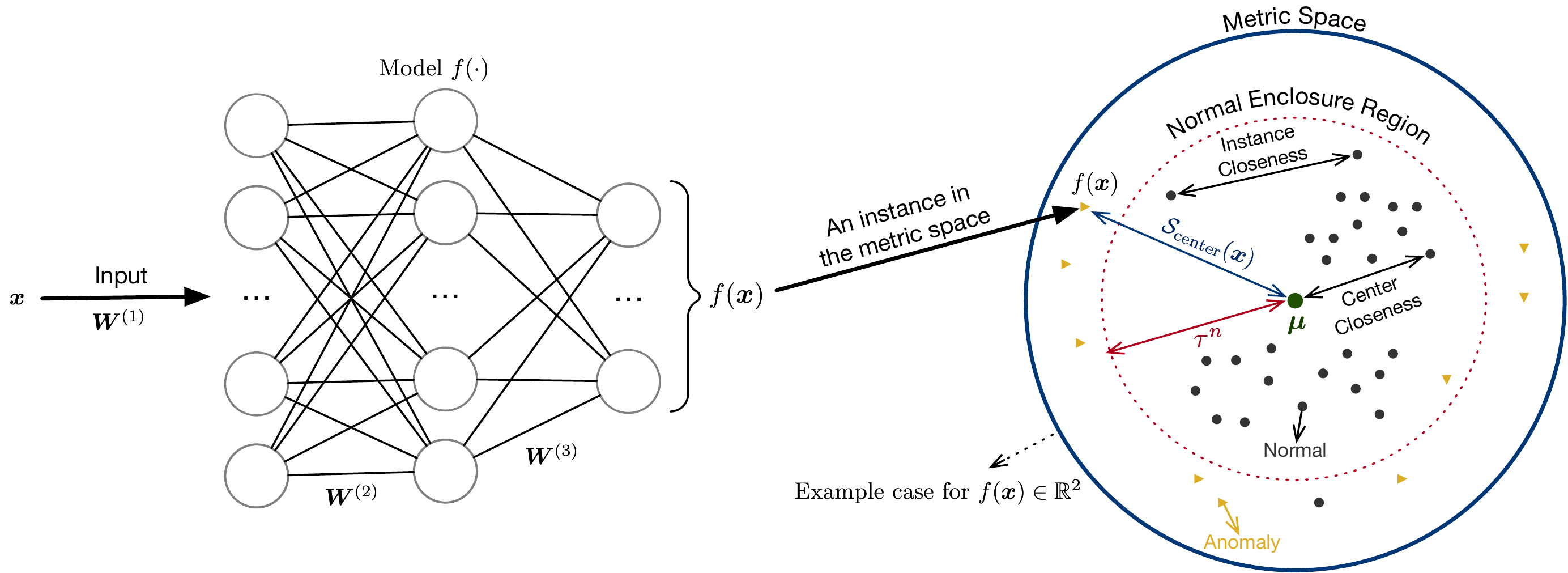}

    \caption{The figure illustrates the introduced method and the terms used. All the model parameters are jointly optimized. As the model $f(\cdot)$, one can use different approaches described in the text as remarks. The figure is best viewed in color.}

    \label{fig:main_figure}

\end{figure*}
In this article, all vectors are column vectors and denoted by boldfaced lowercase letters. All matrices are denoted by boldfaced uppercase letters. We denote the multisets via $\{\cdot\}$. A multiset is a set that can have more than one identical members.

The learner $f$ observes training data $\mathcal{M}=\{\x_i \in \real^d\}_{i=1}^n$, where $n$ is the number of training data. The corresponding ground truth labels $\{d_i \in \{0,1\}\}^n_{i=1}$ are not available during the training in the unsupervised training and only positive labeled data are available in one class training. The label $d_t=0$ represents a normal instance whereas the label $d_t=1$ represents an anomalous instance. Naturally, we consider anomalous instances as not normal instances. We define the scoring function $\S$ for a test datum $\x_t \in \real^d$ with true label $d_t \in \{0,1\}$ as
\begin{align*}
\S(\x_t) &= p(d_t=1 \mid \x_t).
\label{eq:score_def}
\end{align*}
The $\S$ assigns a score representing the degree of outlierness of each instance. Through the assigned scores, we obtain the list of instances ranked by the degree of outlierness. In the ideal scenario, all anomalous instances in the list should have higher score than the normal instances. Through this ranked list, an analyst can choose a domain-specific cutoff threshold $\tau$ to output the most appropriate anomalies~\cite{chandola2009anomaly}. We construct a binary decision function $\hat{d}_t \in \{0,1\}$ via thresholding the scoring function $\S$ by $\tau \in \real$ to make a decision, i.e.,
\begin{equation}
\hat{d}_t =
\begin{cases} 
1  \text{ (Anomaly)}, & \text{ if } \S(\x_t) > \tau\\    0 \text{  (Normal)}, & \text{ if } \S(\x_t) \leq \tau.
\end{cases}
\label{eq:dhat}
\end{equation}
Now, given the learner $f$, the scoring function $\S$, the threshold $\tau$, and the training data $\M$, we decide whether the test instance $\x_t \in \real^d$ is anomalous or not via $\hat{d}_t$. In the following section, we introduce our method by defining the $f$ and the $\S$.
\section{Novel Anomaly Detection Algorithm}
\label{sec:novel_anomaly}

Here, we first describe the supervised deep metric learning framework. We then provide how it will be used in the semi-supervised and the unsupervised setting via novel losses and data distillation method. Lastly, we provide the details of our training method.

\subsection{Conventional Metric Learning}

\label{sec:mahalanobis_metric}

In this section, we describe the traditional Mahalanobis metric learning. Traditional distance metric learning-based methods learn a symmetric positive semidefinite matrix $\vec{M} \in \mathbb{R}^{d \times d}$ to compute the Mahalanobis distance $\D_{\small{\vec{M}}}$ as~\cite{kulis2013metric}:
\begin{align*}
    \D_{\small{\vec{M}}} (\x_i, \x_j)
    =
    \sqrt{(\x_i - \x_j)^T\vec{M}(\x_i - \x_J)},
\end{align*}
where $\x_i \in \mathbb{R}^d$ and $\x_j \in \mathbb{R}^d$.
We can decompose $\vec{M} = \vec{L}^T \vec{L}$ for $\vec{L} \in \mathbb{R}^{q \times d}$ and $q \leq d$ since $\vec{M}$ is a positive semidefinite matrix, we have:
\begin{align*}
    \D_{\small{\vec{M}}} (\x_i, \x_j) &= \sqrt{(\x_i - \x_j)^T\vec{L}^T\vec{L}(\x_i - \x_j)} \\
    &=  \sqrt{(\vec{L}\x_i - \vec{L}\x_j)^T(\vec{L}\x_i - \vec{L}\x_j)} \\
    &= \norm{\vec{L}\x_i - \vec{L}\x_j}_2.
\end{align*}
Instead of learning the $\vec{M}$ with a positive semidefiniteness constraint, we can directly optimize the unconstrained $g(\x)=\vec{L}\x$. The function $g$ is a linear transformation, which projects $\x$ into a (possibly) lower dimensional subspace. Using the linear transformation $g(\x) = \vec{L}\x$, we obtain
\begin{align*}
    \D_{\small{\vec{M}}} (\x_i, \x_j) &= \norm{g(\x_i) - g(\x_j)}_2.
\end{align*}
In the following section, we introduce a deep learning based extension for the metric learning.
\subsection{Deep Metric Learning (DML)}
\label{sec:deep_metric}
Here, we show the transition from the Mahalanobis metric learning to the DML framework. Note that $g$ cannot model the nonlinear relations in the data since it only multiplies the input with $\vec{L}$. We extend the linear mapping $g$ to the nonlinear multi layer neural network by construction a function $f$ due to the superior modeling power of the neural networks~\cite{lecun2015deep}. We denote the number of layers in the neural network as $u$, and the number of neurons in each layer as $v^{(c)}$, where $c \in \{1,2,...u\}$. We also denote the weight as $\vec{W}^{(c)} \in \mathbb{R}^{v^{(c)} \times v^{(c-1)}}$, and the bias as $\vec{b}^{(c)} \in \mathbb{R}^{v^{(c)}}$ for each layer. Let $\vec{h}^{(c)} = \mathrm{tanh}(\vec{W}^{(c)} \vec{h}^{(c-1)} + \vec{b}^{(c)})$ for the layer $c$ and $\vec{h}^{(0)} = \x$. Then, we construct the neural network $f: \mathbb{R}^d \rightarrow \mathbb{R}^w$ through its $u^{\mathrm{th}}$ layer as:
\begin{align*}
    f(\x) &= \vec{h}^{(u)} \\
    &= \mathrm{tanh}(\vec{W}^{(u)}\vec{h}^{(u-1)} + \vec{b}^{(u)}),
\end{align*}
where $\x \in \mathbb{R}^d$ is the input vector, $u$ is the final layer and $\vec{h}^{(u)} \in \real^w$ is the final output.

As shown in the Fig.~\ref{fig:main_figure}, the network $f$ transforms the $d$ dimensional input vector $\x$ into a $w$ dimensional latent vector $f(\x)$. Then, the distance measure between $\x_i$ and $\x_j$ becomes the euclidean distance between the latent representations $\D_f (\x_i, \x_j)$, which can also be seen as the generalized Euclidean distance as:
\begin{align*}
    \D_f (\x_i, \x_j) &= \norm{f(\x_i) - f(\x_j)}^2_2.
\end{align*}
We employ the distance $\D_f (\x_i, \x_j)$ as a dissimilarity measure. We point out that other distance measures such as cosine distance can be used in our framework instead of $\D_f$. We denote the set of the positive pairs $\mathcal{I}^+ \subset \real^d \times \real^d$, which contains only the instances of the same type and the set of the negative pairs $\mathcal{I}^- \subset \real^d \times \real^d$, which contains only the instances of the different type as the following:
\begin{align*}
	\mathcal{I}^+ &= \{(\x_i, \x_j) \mid d_i = d_j \} \\
    \mathcal{I}^- &= \{(\x_j, \x_k) \mid d_j \neq d_k \}.
\end{align*}
The Contrastive loss~\cite{hadsell2006dimensionality} for the supervised training $\mathcal{L}_\mathrm{supervised}$ is given by
\begin{align*}
\mathcal{L}_\mathrm{supervised} = & \frac{1}{\card{\mathcal{I}^+} + \card{ \mathcal{I}^-}} [\sum_{(\x_i, \x_j)\in \mathcal{I}^+ }
\norm{f(\x_i)- f(\x_j)}^2_2
\\
&+\sum_{(\x_j, \x_k)\in \mathcal{I}^- }\max(0,a-\norm{f(\x_j) - f(\x_k)}^2_2)],
\end{align*}
where $a>0$ represents the margin between normal and anomalous instances. Note that the Contrastive loss~\cite{hadsell2006dimensionality} requires the true labels due to its supervised nature. In the next section, we show how the supervised deep metric learning is used in the one-class setting.








\subsection{Metric Space Learning via Instance Closeness Loss}

\label{sec:instance_closeness}

Before introducing the unsupervised framework, we first define our framework for the one-class training setting, i.e., all of the training data are known to be normal. Thus, the positive pairs $\I^+$ and the negative pairs $\I^-$ become equivalent to 
\begin{align*}
\mathcal{I}^+ &= \{(\x_i, \x_j) \mid d_i = 0 \and d_j = 0 \} \,\mathrm{(normal\, pairs)} \\
I^- &= \varnothing \mathrm{(empty)}.
\end{align*}
Since $\I^-$ is empty in one-class setting, minimizing the loss $\Lsupervised$ is equivalent to minimizing the loss $\Lsimilarity$, which is defined as:
\begin{align*}
\Lsimilarity &= \frac{1}{\MC^2}
\sum_{i=1}^{\MC-1}
\sum_{j=i+1}^{\MC}
\norm{f(\x_i) - f(\x_j)}^2_2.
\end{align*}
\begin{remark}
One can directly apply $\Lsimilarity$ to the unsupervised setting following the common practice of assuming that all the training data is normal, similar to the conventional methods~\cite{aggarwal2015outlier}. However, this practice highly degrades the performance of many anomaly detection methods~\cite{chandola2009anomaly} (see Section~\ref{sec:unsupervised_experiments}).
\end{remark}
In the following section, we introduce another loss function that we later use to derive a feasible scoring function for $\Lsimilarity$.

\subsection{Metric Space Learning via Center Closeness Loss}
\label{sec:center_closeness}
We develop the center closeness loss $\Lcenter$, which is equivalent to the $\Lsimilarity$ for the batch training setting.

\begin{lemma}
$\Lsimilarity$ is equivalent to the following statement
\begin{align*}
    \frac{1}{\MC} \sum_{i=1}^{\MC}
    \norm{f(\x_i) - \muu}_2^2,
\end{align*}
where $\muu = \frac{1}{\MC} \sum_{k=1}^{\MC} f(\x_k)$.
\label{prop:center_loss}
\end{lemma}
\begin{proof}[Proof of Lemma~\ref{prop:center_loss}]
\allowdisplaybreaks
\begin{align*}
    \Lsimilarity &= \frac{1}{{\MC}^2} \sum_{i=1}^{\MC-1} \sum_{j=i+1}^{\MC} \norm{f(\x_i) - f(\x_j)}_2^2 \\
    &= \frac{1}{2{\MC}^2} \sum_{i=1}^{\MC} \sum_{j=1}^{\MC} \norm{f(\x_i) - f(\x_j)}_2^2 \\
    &= \frac{1}{2{\MC}^2} \sum_{i=1}^{\MC} \sum_{j=1}^{\MC} \norm{f(\x_i) - \muu + \muu - f(\x_j)}_2^2 \\
    &= \frac{1}{2{\MC}^2} \sum_{i=1}^{\MC} \sum_{j=1}^{\MC} \norm{f(\x_i) - \muu}_2^2 \\
    &\,\,\,\,+ \frac{1}{2{\MC}^2}\sum_{i=1}^{\MC} \sum_{j=1}^{\MC} \norm{f(\x_j) - \muu}_2^2 \\
    &\,\,\,\,- \frac{1}{{\MC}^2} \sum_{i=1}^{\MC} \sum_{j=1}^{\MC} (f(\x_i) - \muu)^T(f(\x_j) - \muu) \\
    &=  \frac{1}{\MC} \sum_{i=1}^{\MC} \norm{f(\x_i) - \muu}_2^2 \\
    &\,\,\,\,- \frac{2}{{\MC}^2} \sum_{i=1}^{\MC} (f(\x_i) - \muu)^T (\sum_{j=1}^{\MC} (f(\x_j) - \muu)). \\
\end{align*}
Since the term at the end $\sum_{j=1}^{\MC} (f(\x_j) - \muu) = \sum_{j=1}^{\MC} f(\x_j) - \sum_{k=1}^{\MC} f(\x_k) = 0$, $\Lsimilarity$ is equal to the following term
\begin{align*}
    \frac{1}{\MC} \sum_{i=1}^{\MC} \norm{f(\x_i) - \muu}_2^2, \\
\end{align*}
where $\muu = \frac{1}{\MC} \sum_{k=1}^{\MC} f(\x_k)$. This concludes the proof of Lemma~\ref{prop:center_loss}.
\end{proof}
\begin{remark}
Lemma~\ref{prop:center_loss} shows that minimizing the distance of instances to the mean of $f$ for the batch is equivalent to the minimizing the similarity between the instances for the batch case. This equivalence explains that we can obtain a compact distribution of the data in the metric space for the normal samples by minimizing the similarity between instances, i.e., by minimizing $\Lsimilarity$.
\end{remark}

To further analyze the result of the Lemma~\ref{prop:center_loss} we define the center closeness loss $\Lcenter$ as:
\begin{align*}
    \Lcenter &=  \frac{1}{\MC} \sum_{i=1}^{\MC} \norm{f(\x_i) - \muu}_2^2 ,
\end{align*}
where $\muu = \frac{1}{\MC} \sum_{k=1}^{\MC} f(\x_k)$. In the next section, we derive scoring functions using the $\Lsimilarity$ and the $\Lcenter$.

\subsection{Anomaly Scoring Functions}

\label{sec:scoring}

Here, we introduce two different scoring functions $\Sdis$ and $\Scenter$ to score the outlierness of the instances. Recall that we define the decision function $\hat{d}_t$ in~\eqref{eq:dhat} via thresholding the scoring function where the scores greater than $\tau$ are labeled as anomalies by our framework. Thus, our goal is to obtain larger scores for the anomalies and smaller scores for the normal instances. $\Lsimilarity$ minimizes the following expectation for the normal instances:
\begin{align}
\E_{\x_t \in \mathcal{M}} \left [\frac{1}{\card{\mathcal{M}}} \sum_{\x_j \in \mathcal{M}} \norm{f(\x_t) - f(\x_j)}^2_2 \mid d_t = 0 \right].
\label{eq:dis_scoring_expectation}
\end{align}
Accordingly, we define the euclidean distance (dissimilarity) for all instances as the scoring function as:
\begin{align*}
    \Sdis(\x_t) = \frac{1}{\card{\mathcal{M}}} \sum_{\x_j \in \mathcal{M}} \norm{f(\x_t) - f(\x_j)}^2_2.
\end{align*}
However, this scoring function requires to calculate the distance to all instances, i.e., $\Theta(\card{\mathcal{M}})$ distance calculations for scoring an instance. The number of training instances are usually too large for most applications such as intrusion detection~\cite{chandola2009anomaly} and the time complexity of the $\Sdis$ grows with the training data size. Thus, employing $\Sdis$ as the scoring function is time-consuming and usually impractical. Since we show that $\Lsimilarity$ is equivalent to the $\Lcenter$ via Lemma~\ref{prop:center_loss}, hence both $\Lsimilarity$ and $\Lcenter$ minimize the following expectation for the normal instances:
\begin{align}
\E_{\x_t \in \M} \left [\norm{f(\x_t) - \muu}_2^2 \mid d_t = 0 \right],
\label{eq:center_scoring_expectation}
\end{align}
where $\muu = \frac{1}{\MC} \sum_{k=1}^{\MC} f(\x_k)$. Similar to the $\Sdis$, we define $\Scenter$ as the euclidean distance to the center of the outputs of the model $f$ as:
\begin{align*}
    \Scenter (\x_t) = \norm{f(\x_t) - \muu}^2_2,
\end{align*}
where $\muu = \frac{1}{\MC} \sum_{k=1}^{\MC} f(\x_k)$ and $f$ is the trained model. As shown in Fig.~\ref{fig:main_figure}, anomalous instances appear farther to the $\muu$ than the normal instances. Algorithm~\ref{algo:dis_center_scoring} summarizes the procedure for the $\Sdis$ as the \textit{dissimilarity} method, and the $\Scenter$ as the \textit{center} method.

\begin{algorithm}[ht!]
	\SetAlgoLined
	\small
	\KwIn{Model $f_\mathrm{best}$, target instance $\vec{x}_t$, training data $\mathcal{M}$ (after validation split), $method$ (name of the scoring function, i.e., $center$ or $dissimilarity$) }
	
	\KwOut{Score $s$ of the target instance $\vec{x}$ }
	    
    \textbf{Build Retrieval Set:} (only at the first run)
    \nl $\mathcal{R} = \{ f_\mathrm{best}(\vec{x}_i) \}_{
    \vec{x}_i \in \mathcal{M}
    }$
    
    \textbf{Calculate Center:} (only at the first run)

    \nl $\muu = \frac{1}{\mid \mathcal{M} \mid} \sum_{\vec{r} \in \R} \vec{r}$

	\textbf{Calculate score:}

	\nl \uIf{$method =$ dissimilarity}{
	    \nl $\vec{t} = \fbest (\x_t)$ \\
	    \nl $s = \sum_{\vec{r} \in \R} \norm{\vec{t} - \vec{r}}_2^2$
	}
	\nl \uElseIf{$method = $ center}{
	    
	\nl $s = \norm{\fbest (\x_t) - \muu}_2^2$
	}

	\caption{Dissimilarity \& Center Based Scoring}
    \label{algo:dis_center_scoring}
\end{algorithm}

\begin{remark}
Notice that $\Scenter$ only requires to calculate the distance to the center as shown in Fig.~\ref{fig:main_figure}, thus, its time and memory complexities are $\Theta(1)$ unlike $\Sdis$ with $\Theta(\MC)$ time and memory complexity. Since both of the $\Lsimilarity$ and the $\Lcenter$ minimizes the expectation in~\eqref{eq:center_scoring_expectation}, they can be used interchangeably with the $\Scenter$.
\end{remark}
In the next section, we introduce a self-supervised data distillation method that iteratively cleans the training data to obtain more reliable training data.

\subsection{Self-Supervised Data Distillation}

\label{sec:data_distillation}

In this section, we introduce a data distillation method for the $\Lsimilarity$ and the $\Lcenter$ via self-supervision that improves our method in the unsupervised setting.

\begin{definition}{SmallestK($A$,$K$)}
returns the multiset of $K$ elements with the lowest values in $A$.
\end{definition}

Anomaly detection methods often assume nearly all the training data as normal and provide good performance when the assumption holds and their performance significantly deteriorates if this assumption in wrong~\cite{chandola2009anomaly}. 
To remedy this problem, we introduce a self-supervision based approach to distill the data between epochs while training. In this way, we reduce the anomaly rate of the training data throughout. After each epoch and in the beginning, we score all the training instances through $\Scenter$. Using these scores, we distill the data by selecting $\rho^n$ of the instances that are more likely to be normal than the rest of the training data according to our model. At the beginning of each epoch, we construct the set of distilled training data $\P$ as:
\begin{align*}
    \P &= \{ \x_i \mid \norm{ f(\x_i) - \muu}_2 \leq \tau^n \}_{\x_i \in \M},
\end{align*}
where
\begin{align*}
    \muu &= \frac{1}{\MC} \sum_{k=1}^{\MC} f(\x_k) \\
    \F &= \{ \norm{f(\x_i) - \muu }_2 \}_{\x_i \in \M} \\
    \tau^n &= \max (\mathrm{SmallestK}(\F, K=\rho^n \MC ))
\end{align*}
and $\rho^n \in (0,1]$ is a hyperparameter that defines the ratio of the instances in $\M$, which we use to train our model for each epoch. As shown in Fig.~\ref{fig:main_figure}, we select $\tau^n$ to enclose $\rho^n$ of all instances in the training data. Then, we use those instances to build $\P$. Through training on $\P$ on a particular epoch, we build an enclosure region over the training data so that we consider the instances inside the region as normal during the epoch.



\begin{remark}
Recall that $\Scenter$ scores all the training data in a short time since its time complexity is $\Theta(1)$ Accordingly, the data distillation with $\rho^n < 1$ further reduces the running time at each epoch since it decreases the number of elements in the loss.
\end{remark}
\begin{remark}
Defining the hyperparameter $\rho^n$ in terms of the ratio improves interpretability instead of directly defining $\tau^n$ as a hyperparameter. It also removes the requirement to choose a different $\tau^n$ manually for each dataset.
\end{remark}
\begin{remark}
We initialize the model $f$ via uniform Glorot initialization~\cite{glorot2010understanding}. Although the model is not trained in the beginning, it still can extract features from the data and projects the data similarly to the random projection. Random projection is applied before applying the anomaly detection methods and shown to be effective in anomaly detection tasks~\cite{aggarwal2015outlier}. Our scoring functions are capable of scoring anomalies even with the randomly projected data. Thus, we apply the data distillation before the first epoch, too.

\end{remark}
In the next section, we describe the construction of the mini-batches for the training.







\subsection{Mini-batch Construction}
\label{sec:mini_batch_construction}
We train our model through mini-batches of instances. We randomly select a small subset of size $b$ without replacement from the training data $\M$. We repeat this process for iterations until we select all instances in $\M$ for each epoch. For the construction of the mini-batches, we use the distilled training data $\P$ after shuffling to facilitate the random sampling. Given that $b$ is the mini-batch size and $\muu = \frac{1}{\MC} \sum_{k=1}^{\MC} f(\x_k)$ is calculated at the beginning of each epoch, we construct the $m^\mathrm{th}$ mini-batch $\B^{(m)}$ of size $b$, and the corresponding set of distances $\D^{(m)}$ for the instance closeness loss as:
\begin{align*}
    \B^{(m)} &= \{(\x_i, \x_j) \mid i < j < mb \}_{i=(m-1)b+1}^{mb} \\
    \D^{(m)} &= \{ \norm{f(\x_i) - f(\x_j)}_2 \}_{(\x_i,\x_j) \in \B^{(m)}.}
\end{align*}
Similarly for the center closeness loss, we define the mini-batch  $\B^{(m)}$ and the corresponding set of distances $\D^{(m)}$ as:
\begin{align*}
    \B^{(m)} &= \{ \x_i \}_{i=(m-1)b+1}^{mb} \\
    \D^{(m)} &= \{ \norm{f(\x_i) - \muu}_2 \}_{\x_i \in \B^{(m)}.}
\end{align*}
In the following section, we present the hard normal mining method.

\subsection{Online Hard Normal Mining}

\label{sec:hard_mining}
Here, we introduce the hard normal mining method for our framework that selects the specific instances for training, which are more likely to improve the model~\cite{harwood2017smart}.

The idea behind hard mining is to select the samples that are challenging so that the model does not overfit to the easiest instances in the training data~\cite{harwood2017smart}. We perform the hard mining on the mini-batches, i.e., in online manner, to avoid the calculation of the distances for all pairs in the training data. Recall that we define the normal enclosure region in Section~\ref{sec:data_distillation}. We construct mini-batches using the instances in the normal enclosure region, i.e., the distilled training data. Thus, the selected instances are less likely to be anomalies. 

\begin{definition}{LargestK($A$,$K$)}
returns the multiset of $K$ elements with the largest values in $A$.
\end{definition}

We introduce an online method for the hard mining of pairs relying on the normal enclosure region. We only incorporate the $\rho^h$ of the hardest instances in each mini-batch that are inside the normal enclosure region. For instance closeness loss, we define the hardest instances as the farthest pairs of instances in the enclosure region. For center closeness loss, we define the hardest instances as the farthest points to the center in the normal enclosure region. Through the already calculated set of distances in mini-batch $\D^{(m)}$, we select the set of hardest distances in the mini-batch $m$ as:
\begin{align*}
    \Q^{(m)} = \{r \mid r \geq \tau^h \}_{r \in \D^{(m)}},
\end{align*}
where
\begin{align*}
    \tau^h = \min(\mathrm{LargestK}(\D^{(m)}, K=\rho^h \card{\D^{(m)}}))
\end{align*}
and $\rho^h \in (0, 1]$ is a hyperparameter that defines the ratio of the distances in $\D^{(m)}$, which we use to train our model for each mini-batch.

\begin{remark}
The online hard normal mining with $\rho^h < 1$ reduces the number of elements in the loss, therefore, it decreases the running time of our method.
\end{remark}

We describe the training details of our model in the next section.
\subsection{Training of the Model}
In this section, we provide the details of the training and combine introduced methods described in the previous sections.

Using $\Q^{(m)}$, we unify the calculation of both losses and combine them with the weight decay regularization into $\L^{(m)}$ as:
\begin{align*}
    \L^{(m)}= \frac{1}{\card{\Q^{(m)}}} \sum_{r \in \Q^{(m)}} r +\frac{\lambda}{2} \sum_{c=1}^u \norm{\vec{W}^{(c)}}_F^2,
\end{align*}
where $u$ is the number of layers in the network, $\vec{W}^{(c)}$ is the weight matrix of the layer $c$, and $\lambda$ is the weight decay regularization hyperparameter. After each mini-batch's forward pass, we update our model $f$. For each layer $c=\{1,...u\}$, we update the layer weight matrix $\W^{(c)}$ and bias vector $\b^{(c)}$ using $\frac{\partial \mathcal{L}}{\partial\vec{W}^{(c)}}$, and $\frac{\partial \mathcal{L}}{\partial\vec{b}^{(c)}}$, respectively. We use Adam's algorithm~\cite{kingma2014adam} for parameter optimization of the model $f$. We calculate the center $\muu$ at the beginning since it is not feasible to calculate it after each weight update. 

We also adapt early stopping and the best model selection through the validation set $\V$. We select the best model $f_\mathrm{best}$ as the model with the lowest validation loss. If validation loss does not improve for $e^v$ epochs, we stop the training and set $f_{\mathrm{best}}$ as the best model, i.e., the model with the lowest validation loss. Note that since our method is \textbf{unsupervised} and the losses do not require any labels, we do not use any labels from the validation set, too.

\begin{algorithm}[htbp]
	\SetAlgoLined
	\small
	\KwIn{Model $f$, training data $\mathcal{M}$ }
	
	\KwOut{Trained model $f_\mathrm{best}$ with the lowest validation error}
	\textbf{Parameters:} $n$ (number of epochs), $b$ (batch size), $\rho^n$ (ratio of data distillation), $\rho^h$ (ratio of online hard normal mining), $\rho^v$ (ratio of validation split), $e^v$ (number of epochs without validation improvement), $\lambda$ (weight decay regularization parameter, $b$ (mini-batch size)\\
	
	\nl \textbf{Initialize } $\{\W^{(c)}\}_{c=1}^u$ via Uniform Glorot~\cite{glorot2010understanding} initialization
	
	\nl \textbf{Initialize: } $v = \infty$ (lowest validation loss)

	\nl \textbf{Initialize: } $z = e^v$ (number of epochs before early stopping)
	
	\nl Separate $\rho^v$ of $\mathcal{M}$ into the validation set  $\mathcal{V}$
	
	\textbf{Training:}
	
	\nl \For{epoch $e=1$ to $n$}{
	    
	    \nl \uIf{z = 0}{
	    
	    Early stopping
	    
	    \nl \textbf{break}
	    
	    }

	    \textbf{Data Distillation:}
	    
        \nl $\muu = \frac{1}{\MC} \sum_{k=1}^{\MC} f(\x_k)$
        
        \nl $\F = \{ \norm{f (\x_i) - \muu }_2 \}_{\x_i \in \M}$

        \nl $\tau^n = \max (\mathrm{SmallestK}(\F, K=\rho^n \MC ))$
         
        \nl $\P = \{ \x_i \mid \norm{ f(\x_i) - \muu}_2 \leq \tau^n \}_{\x_i \in \M}$

		\textbf{Construct Mini-batches:}
		
        \nl Shuffle $\P$ with random sampling
        
        \nl \For{mini-batch $m=1$ to $\left \lceil \frac{\mid \mathcal{P} \mid}{b} \right \rceil$}{
            
            \nl \uIf{$loss = $ instance\_closeness}{
                \nl $\B^{(m)} = \{(\x_i, \x_j) \mid i < j < mb \}_{i=(m-1)b+1}^{mb}$
            }
            
            \nl \uElseIf{$loss = $ center}{
                \nl $\B^{(m)} = \{ \x_i \}_{i=(m-1)b+1}^{mb}$
            }
            
        }

        \nl \For{mini-batch $m=1$ to $\left \lceil \frac{\mid \mathcal{P} \mid}{b} \right \rceil$}{
            \textbf{Forward Pass:}

            \nl \uIf{$loss = $ instance\_closeness}{
                \nl $\D^{(m)} = \{ \norm{f(\x_i) - f(\x_j)}_2 \}_{(\x_i,\x_j) \in \B^{(m)}}$
            }
            
            \nl \uElseIf{$loss = $ center}{
                \nl $\D^{(m)} = \{ \norm{f(\x_i) - \muu}_2 \}_{\x_i \in \B^{(m)}}$
            }
            
            \textbf{Online Hard Normal Mining:}
            
            \nl $    \tau^h = \min(\mathrm{LargestK}(\D^{(m)}, K=\rho^h \card{\D^{(m)}}))$
            
            \nl $ \Q^{(m)} = \{r \mid r \geq \tau^h \}_{r \in \D^{(m)}}$
            
            \textbf{Calculate Loss: }
            
            \nl $\L^{(m)} = \frac{1}{\card{\Q^{(m)}}} \sum_{r \in \Q^{(m)}} r +\frac{\lambda}{2} \sum_{c=1}^u \norm{\vec{W}^{(c)}}_F^2$

            \textbf{Backward Pass:}
            
            \nl \For{layer $c=1$ to $u$}{
                \nl Assume $\muu$ is scalar, i.e., $\frac{\partial \muu}{\partial\vec{W}^{(c)}} = \frac{\partial \muu}{\partial\vec{b}^{(c)}} = 0$.
    
                \nl Update layer weights $\vec{W}^{(c)}$ via backpropagation using $\frac{\partial \mathcal{L}^{(m)}}{\partial\vec{W}^{(c)}}$
                
                \nl Update layer bias terms $\vec{b}^{(c)}$ via backpropagation using $\frac{\partial \mathcal{L}^{(m)}}{\partial\vec{b}^{(c)}}$
            }
            
            \textbf{Validation:}
            
            \nl \uIf{$loss = $ instance\_closeness}{
                \nl $\D_v = \{ \norm{f(\x_i) - f(\x_j)}_2 \}_{(\x_i,\x_j) \in \V \times \V}$
            }
            
            \nl \uElseIf{$loss = $ center}{
                \nl $\D_v = \{ \norm{f(\x_i) - \muu}_2 \}_{\x_i \in \V }$
            }
            
            \nl $\mathcal{L}_v = \sum_{r \in \mathcal{D}_v } r $
            
            \nl \uIf{$\mathcal{L}_v < v$}{
                
                \nl $v = \mathcal{L}_v$
                
                \nl $z = e^v$
                
                \nl $f_\mathrm{best} = f$
            }
            \nl \uElse{
                \nl $z = z - 1$
            }
        }
        
    }
	\caption{Complete Training Procedure}
    \label{algo:training}
\end{algorithm}

Algorithm~\ref{algo:training} summarizes the training of the whole model for both center closeness and instance closeness losses.

\begin{remark}

We avoid the model collapse, i.e., $f(\x)=0$ $\forall \x$, through validation, mini-batch-training and hard normal mining.

\end{remark}


\begin{remark}
we can readily extend our work to the CNNs~\cite{krizhevsky2012imagenet} for anomaly detection on images since they are shown to be successful in image-related tasks. We are given a training set of RGB images $\{ \vec{m}_i \in \real^{3\times s \times h} \}_{i=1}^n$, where $s$ is the width, and $h$ is the height. We first map the image to $w$ dimensional latent vector via a convolutional network followed by flattening, i.e., $c(\cdot): \real^{3\times s \times h}\rightarrow \real^d$, which may also contain a feed-forward extension at the end. Instead of the $f(\cdot)$ that processes only the vectors, we use $c(\vec{m})$ for $\vec{m} \in \real^{3\times s \times h}$.
\end{remark}

\begin{remark}
we can readily extend our framework to the RNNs~\cite{elman1990finding} for anomaly detection on time series due to their ability to store the past information. The RNN outputs $\vec{h}_t = \mathrm{tanh}(\vec{U} \x + \vec{V} \vec{h}_{t-1})$, where $t$ is the time step, $w$ is the metric space size, weight $\vec{U} \in \real^{w \times d}$, weight $\vec{V} \in \real^{w \times w}$, $\vec{h}_t \in \real^w$ and $\vec{h}_0 \in \real^w$ is the initial hidden state vector. Instead of the $f(\cdot)$, we directly use $\vec{h}_t$. Notice that the network can be extended to the multiple layers and combined with other layers such as feed-forward layer as long as it outputs $w$ dimensional latent vector.
\end{remark}

In the following section, we demonstrate our experiments and their details.



\section{Simulations}
\label{sec:experiments}

In this section, we demonstrate the anomaly detection performance of our method on various real-world datasets and compare it with the state-of-the-art methods. We also analyze the behavior of the method with respect to its parameters and perform an ablation study.

\subsection{Evaluation Methodology}
\label{sec:evaluation_methodology}
In this section, we describe the evaluation metrics, the cross validation strategy and the training settings.


To compare the performance on the anomaly detection datasets, we use the Area Under Curve of the Receiver Operating Characteristics (AUC of the ROC), which is commonly used in the anomaly detection literature~\cite{aggarwal2015outlier}. We apply the 3-fold stratified cross-validation. We split the data into 3-folds with equal anomaly rates and repeat the experiment by choosing one fold as the test set and the rest as the training set. We repeat this process for each of the three splits. To reduce the effect of randomness in our experiments, we repeat all cross-validation experiments three times with different random seeds. Hence, we report the mean score of the nine experiments for every method-dataset pair.

We report the scores for the three different settings for all experiments. For the \textit{seen} setting, we train and test on the training set. For the \textit{unseen} setting, we train on the training set and test on the test set. For the \textit{one-class} setting, we manually remove anomalies from the training set and evaluate on the test set. In the following section, we describe the methods used in our comparisons and their implementation details.

\subsection{Compared Methods and Implementation Details}
\label{sec:implementation}
We have introduced the Anomaly Detection with Deep Metric Learning (ADDML) method in Section~\ref{sec:novel_anomaly}. We compare our method with Deep Autoencoder (DAE)~\cite{aggarwal2015outlier}, Isolation Forest (IF)~\cite{liu2008isolation}, Robust Principal Component Classifier (PCC)~\cite{shyu2006principal}, One-Class Support Vector Machine (OC-SVM)~\cite{scholkopf2001estimating}, Histogram Based Outlier Scoring (HBOS)~\cite{goldstein2012histogram}. We select the hyperparameters as suggested in the particular papers~\cite{liu2008isolation,shyu2006principal,scholkopf2001estimating,goldstein2012histogram}. We use IF with $100$ trees, and $256$ samples ($214$ for \textit{Glass} dataset) for each tree. We use OC-SVM with Gaussian kernel and $\nu=0.5$. We use HBOS with $5$ bins. For all experiments, we normalize the data via subtracting the mean of the training data and dividing by the standard deviation of the training data. 

For ADDML, which is our method, we use the same hyperparameters for all experiments except the \textit{one-class} setting, where we use the normal ratio $\rho^n=1$ due to the results of the ablation study in Section~\ref{sec:ablation}. We use $\rho^n=\frac{2}{3}$ for the \textit{seen}, and the \textit{unseen} settings. We use the instance closeness loss with the center distance scoring function $\Scenter$ since the dissimilarity based scoring function $\Sdis$ is infeasible for large number of training instances. We use the hard mining ratio of $\rho^h=\frac{1}{3}$. We use metric space dimension $w=64$, which is also the latent space dimension for the DAE. For the DAE and the ADDML, we use Adam~\cite{kingma2014adam} optimizer with the default learning rate $10^{-3}$ and the weight decay $\lambda=10^{-5}$. We also use the $\rho^v=0.1$ of the training set as a validation set. We train the models for $n=50$ epochs with early stopping with $e^v=5$ patience. In the next section, we specify the datasets and their details. 

\subsection{Datasets}
Table~\ref{tab:unsupervised_scores} shows the datasets along with the number of dimensions, the number of instances and the percentage of the outliers in each dataset we use. These are the real life datasets widely used in the anomaly detection experiments in the literature~\cite{manzoor2018xstream,liu2008isolation}, which are obtained from~\cite{Rayana:2016}. In the following section, we demonstrate the effect of the components.
\subsection{Ablation Study}
\label{sec:ablation}

Here, we perform an ablation study to observe the effect of the data distillation and the hard mining components.

Table~\ref{tab:ablation} shows the AUC scores on the \textit{Letter} dataset when the components added one by one. We apply the evaluation methodology described in Section~\ref{sec:evaluation_methodology}. We first evaluate the effect of the instance closeness, and center closeness losses without the data distillation and the hard mining components by setting $\rho^n=1$ and $\rho^h=1$. Then, we add the data distillation component via setting $\rho^n=\frac{2}{3}$. Lastly, we add the hard mining component by setting $\rho^h=\frac{1}{3}$.

\begin{table}[h!]

\caption{The AUC under ROC scores for the \textit{Letter} Dataset When the Components are Added Gradually}

\begin{tabular}{@{}lccc@{}}

\toprule

Components          & Seen & Unseen & One-Class \\ \midrule

Instance Closeness Loss  &    77.42 & 77.96 & 81.37          \\
+ Data Distillation  &    80.51 & 80.17 & 80.83          \\
+ Hard Normal Mining  &   \textbf{81.49}& \textbf{81.17} & \textbf{82.33}        \\ \midrule
Center Closeness Loss  &    79.80 & 79.63 & 78.87          \\
+ Data Distillation  &    77.14 & 77.80 & 78.53          \\ 
+ Hard Normal Mining  &    77.20 & 78.22 & 79.55          \\ \bottomrule 

\end{tabular}
\label{tab:ablation}
\end{table}

When we include the data distillation to the instance closeness loss, the AUC increases by ~3\% (absolute) in the seen setting and ~2\% (absolute) in the unseen setting. The AUC also slightly increases in the seen and unseen settings of the center closeness loss when we add the data distillation component. For both losses in the one-class setting, the performance decreases as expected when we add the data distillation component. Recall that the motivation behind the data distillation is to remove the anomalies from the training data so that the model would not overfit to those samples. However, there are no anomalies in the one-class setting. Thus, the performance reduces as the number of training instances at each epoch. The AUC slightly increases when we include the hard mining component for both losses and in all settings. The instance closeness loss is more accurate compared to the center closeness loss without the data distillation and the hard normal mining. Moreover, the instance closeness loss benefits significantly more by the data distillation and the hard normal mining components compared to the center closeness loss. Thus, in further experiments, we only analyze the instance closeness loss for fairness. 

\subsection{Unsupervised and One-Class Experiments}
\label{sec:unsupervised_experiments}
In this section, we compare our model with the state-of-the-art unsupervised anomaly detection methods from different paradigms. We use the evaluation and the cross-validation method described in Section~\ref{sec:evaluation_methodology}. We use the parameters described in Section~\ref{sec:implementation}.

\begin{table*}[t!]
\caption{Average AUC Scores in Percentage (\%) Over 9 Runs Per Dataset For Each Algorithm}

\centering
\begin{tabular}{@{}lrrrrcccccc@{}}
\toprule
Dataset & \#Instances & \#Dims & Outlier\% & Setting

& DAE& IF& PCC& OC-SVM& HBOS& ADDML\\ \midrule
\multirow{3}{*}{Gisette}&            \multirow{3}{*}{3850} & \multirow{3}{*}{4971} & \multirow{3}{*}{9.09} & Seen& 74.37& 49.24& 74.23& \textbf{75.02}& 35.60& 73.40\\ 
& & & & Unseen& 73.76& 49.14& 73.63& 74.24& 35.64& \textbf{79.21}\\ 
& & & & One-Class& 81.94& 48.96& 81.86& 82.06& 37.48& \textbf{84.22}\\ \midrule 
\multirow{3}{*}{Glass}&            \multirow{3}{*}{214} & \multirow{3}{*}{9} & \multirow{3}{*}{4.21} & Seen& 60.56& 71.72& 45.35& 58.81& 70.61& \textbf{73.57}\\ 
& & & & Unseen& 56.91& 70.28& 44.20& 53.13& 67.12& \textbf{74.78}\\ 
& & & & One-Class& 59.13& \textbf{74.89}& 47.14& 62.57& 71.68& 71.73\\ \midrule 
\multirow{3}{*}{Ionosphere}&            \multirow{3}{*}{351} & \multirow{3}{*}{33} & \multirow{3}{*}{35.90} & Seen& 84.99& 85.65& 79.24& 85.14& 65.46& \textbf{92.64}\\ 
& & & & Unseen& 85.45& 85.44& 79.76& 84.29& 65.77& \textbf{93.29}\\ 
& & & & One-Class& 90.96& 91.01& 89.55& \textbf{95.93}& 75.00& 95.81\\ \midrule 
\multirow{3}{*}{Isolet}&            \multirow{3}{*}{4886} & \multirow{3}{*}{617} & \multirow{3}{*}{7.96} & Seen& 51.23& 54.61& 51.31& 55.59& 54.07& \textbf{59.33}\\ 
& & & & Unseen& 51.30& 54.86& 51.36& 55.66& 54.04& \textbf{59.50}\\ 
& & & & One-Class& 52.10& 55.55& 52.11& 56.94& 54.61& \textbf{61.16}\\ \midrule 
\multirow{3}{*}{Letter}&            \multirow{3}{*}{1600} & \multirow{3}{*}{32} & \multirow{3}{*}{6.25} & Seen& 50.46& 63.54& 51.77& 59.81& 57.06& \textbf{81.49}\\ 
& & & & Unseen& 50.80& 62.77& 52.08& 60.00& 55.93& \textbf{81.17}\\ 
& & & & One-Class& 51.66& 62.21& 52.66& 61.56& 60.08& \textbf{81.50}\\ \midrule 
\multirow{3}{*}{Madelon}&            \multirow{3}{*}{1430} & \multirow{3}{*}{500} & \multirow{3}{*}{9.09} & Seen& 50.47& 51.93& 50.50& 50.48& \textbf{53.28}& 51.57\\ 
& & & & Unseen& 50.60& 51.14& 50.65& 50.71& \textbf{53.00}& 52.46\\ 
& & & & One-Class& 51.01& 51.99& 51.03& 51.10& 52.85& \textbf{53.77}\\ \midrule 
\multirow{3}{*}{Magic-Telescope}&            \multirow{3}{*}{13283} & \multirow{3}{*}{10} & \multirow{3}{*}{7.16} & Seen& 69.57& \textbf{77.85}& 68.84& 72.56& 73.97& 74.65\\ 
& & & & Unseen& 69.58& \textbf{77.85}& 68.88& 72.55& 73.83& 74.58\\ 
& & & & One-Class& 70.36& 77.53& 74.83& 73.93& 73.26& \textbf{80.80}\\ \midrule 
\multirow{3}{*}{Mnist}&            \multirow{3}{*}{7603} & \multirow{3}{*}{100} & \multirow{3}{*}{9.21} & Seen& 85.85& 79.78& 84.88& 85.00& 65.68& \textbf{85.89}\\ 
& & & & Unseen& 85.80& 79.81& 84.81& 84.98& 65.64& \textbf{86.43}\\ 
& & & & One-Class& 90.86& 86.36& 90.35& 90.80& 70.33& \textbf{91.25}\\ \midrule 
\multirow{3}{*}{Musk}&            \multirow{3}{*}{3062} & \multirow{3}{*}{166} & \multirow{3}{*}{3.17} & Seen& \textbf{100.00}& 99.92& 99.99& \textbf{100.00}& \textbf{100.00}& \textbf{100.00}\\ 
& & & & Unseen& \textbf{100.00}& 99.91& 99.98& \textbf{100.00}& \textbf{100.00}& \textbf{100.00}\\ 
& & & & One-Class& \textbf{100.00}& 95.88& \textbf{100.00}& \textbf{100.00}& \textbf{100.00}& \textbf{100.00}\\ \midrule 
\multirow{3}{*}{Pendigits}&            \multirow{3}{*}{6870} & \multirow{3}{*}{16} & \multirow{3}{*}{2.27} & Seen& 93.69& \textbf{94.98}& 92.49& 93.16& 92.51& 86.91\\ 
& & & & Unseen& 93.68& \textbf{94.80}& 92.48& 93.13& 92.52& 86.16\\ 
& & & & One-Class& 94.24& 96.71& 94.12& 96.46& 93.59& \textbf{97.93}\\ \midrule 
\multirow{3}{*}{Satellite}&            \multirow{3}{*}{6435} & \multirow{3}{*}{36} & \multirow{3}{*}{31.64} & Seen& 61.87& 70.72& 62.88& 66.36& 68.68& \textbf{80.88}\\ 
& & & & Unseen& 61.84& 70.58& 62.86& 66.40& 68.34& \textbf{80.89}\\ 
& & & & One-Class& 67.73& 79.47& 69.57& 75.71& 77.85& \textbf{81.92}\\ \midrule 
\multirow{3}{*}{Satimage-2}&            \multirow{3}{*}{5803} & \multirow{3}{*}{36} & \multirow{3}{*}{1.22} & Seen& 97.77& 99.41& 97.36& 99.67& 97.74& \textbf{99.87}\\ 
& & & & Unseen& 97.79& 99.45& 97.40& 99.67& 97.81& \textbf{99.88}\\ 
& & & & One-Class& 97.93& 99.31& 97.58& 99.70& 97.25& \textbf{99.84}\\ \midrule 
\multirow{3}{*}{Shuttle}&            \multirow{3}{*}{49097} & \multirow{3}{*}{9} & \multirow{3}{*}{7.15} & Seen& 99.00& \textbf{99.66}& 99.06& 99.17& 98.19& 99.14\\ 
& & & & Unseen& 98.99& \textbf{99.67}& 99.05& 99.17& 98.42& 99.20\\ 
& & & & One-Class& 99.36& 99.61& 99.44& 99.62& 98.71& \textbf{99.96}\\ \midrule 
\multirow{3}{*}{Vowels}&            \multirow{3}{*}{1456} & \multirow{3}{*}{12} & \multirow{3}{*}{3.43} & Seen& 57.25& 76.28& 59.67& 77.84& 63.38& \textbf{85.21}\\ 
& & & & Unseen& 56.95& 75.44& 59.09& 77.76& 63.34& \textbf{84.36}\\ 
& & & & One-Class& 59.22& 76.80& 61.20& 82.30& 63.82& \textbf{90.90}\\ \midrule 
 \midrule\multirow{3}{*}{Average AUC} &            \multirow{3}{*}{-} & \multirow{3}{*}{-} & \multirow{3}{*}{-} & Seen & 74.08& 76.81& 72.68& 77.04& 71.16& \textbf{81.75}\\ & & & & Unseen& 73.82& 76.51& 72.59& 76.55& 70.81& \textbf{82.28}\\ & & & & One-Class& 76.18& 78.31& 75.82& 80.62& 73.32& \textbf{85.06}\\ 

\bottomrule
\end{tabular}
\label{tab:unsupervised_scores}
\end{table*}

Among five methods, the ADDML achieves the best performance for 9 times in the seen setting, 10 times in the unseen setting, and for 12 times in the one-class setting out of the 14 datasets. On average, the ADDML improves the AUC in absolute between 4.71\%-10.59\% for the seen setting, 5.73\%-11.47\% for the unseen setting, and 4.44\%-11.74\% for the one-class setting. Thus, our method performs significantly better than the compared well-known methods in the literature.

The ADDML achieves the highest scores in all settings for all datasets with dimensions higher than or equal to 100, i.e., \textit{Mnist}, \textit{Musk}, \textit{Gisette} (except the seen setting), \textit{Madelon} (except the seen and the unseen settings), and \textit{Isolet}. In the seen and the unseen settings, which contains outliers in the training data, the ADDML achieves the highest score for the datasets higher than 10\% outlier rate, i.e., the \textit{Ionosphere}, and the \textit{Satellite}. These results demonstrate that our method achieves outstanding results compared to the well-known methods in the literature when the number of dimensions is high, or when the outlier rate is high.

Similar to the other methods, the ADDML achieves a higher or equal score in the one-class setting compared to its seen and \textit{Unseen} settings for 12 out of the 14 datasets. Furthermore, when the one-class score is lower than the seen setting or the unseen setting, the score of the one-class setting is very close to the other settings. Thus, all methods better model the data as the number of anomalies decrease. In the following section, we analyze the data distillation hyperparameter.

\subsection{Analysis of the Data Distillation Hyperparameter $\rho^n$}
\label{sec:data_distillation_analysis}

Here, we analyze the normal ratio parameter $\rho^n$ for the data distillation, which defines the ratio of the enclosed instances in the normal enclosure region. Recall that only $\rho^n$ of the instances, which are likely to be normal, are used at each epoch to train our method. We use the instance closeness loss with $\rho^h=\frac{1}{3}$ as in Section~\ref{sec:unsupervised_experiments} for all experiments. We use the evaluation and the cross-validation method described in Section~\ref{sec:evaluation_methodology}.

\begin{figure}[htbp]
    \centering
    \includegraphics{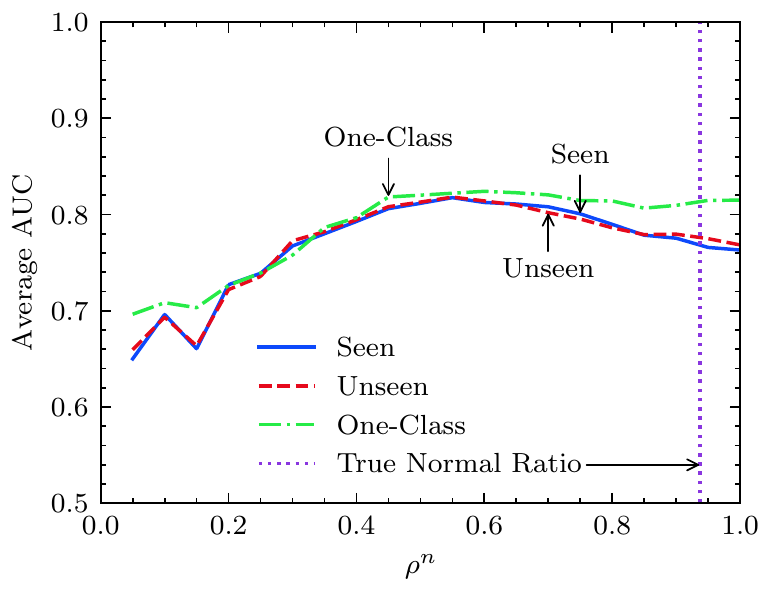}
    \caption{The AUC scores for the \textit{Letter} dataset with respect to the normal ratio $\rho^n \in [0.05,1.00]$ with spacing $0.05$ for $\rho^n$ on the x-axis.}
    \label{fig:rho_n_analysis}
\end{figure}

Fig.~\ref{fig:rho_n_analysis} illustrates the AUC for the \textit{Letter} dataset with respect to the different $\rho^n$ values. In the one-class setting, we do not achieve a significant improvement by changing $\rho^n$. Hence, it is stable, i.e., at most $\approx 2\%$ absolute change in AUC, for $\rho^n>0.5$ in the one-class setting showing that our method does not require hand tuning and robust. Recall that our motivation for data distillation is to clean the training data from anomalies. Thus, the stableness with respect to the $\rho^n$ is advantageous when one does not know whether training data contains anomalies. Our method provides lesser performance for seen and unseen settings when there is no or less data distillation, i.e., $\rho^n \approx 1.0$. As $\rho^n$ decreases, the performance increases since the method use the instances that are more likely to be normal. Thus, the model becomes more robust to the outliers. The performance reaches its peak at near 0.5 instead of the true normal rate, which is 0.9375. Naturally, since our method is unsupervided, it distills the data based on self-supervision instead of the true labels, which is expected. The scores of the ADDML for all $\rho^n$ values in Fig.~\ref{fig:rho_n_analysis} are higher than the scores of the other methods defined in Section~\ref{sec:unsupervised_experiments} for the \textit{Letter} dataset. 

\subsection{Analysis of the Hard Normal Mining Hyperparameter $\rho^h$}

In this section, we analyze the hard normal mining hyperparameter $\rho^h$, which defines the ratio of the largest distances to be included in the loss for the weight updates. As in Section~\ref{sec:unsupervised_experiments}, we use the instance closeness loss with $\rho^n=\frac{2}{3}$ for the seen and the unseen settings and $\rho^n=1$ for the one-class setting for all experiments. We use the evaluation and the cross-validation method described in Section~\ref{sec:evaluation_methodology}.

\begin{figure}[htbp]
    \centering
    \includegraphics{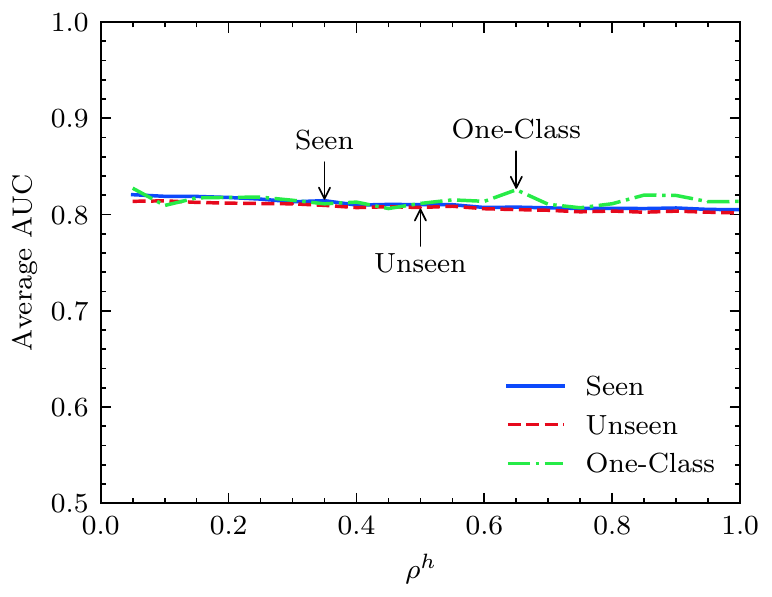}
    \caption{The AUC scores for the \textit{Letter} dataset with respect to hard mining ratio $\rho^h \in [0.05,1.00]$ with spacing $0.05$ for $\rho^h$ on the x-axis.}
    \label{fig:rho_h_analysis}
\end{figure}

Fig.~\ref{fig:rho_h_analysis} illustrates the AUC for the \textit{Letter} dataset with respect to the different $\rho^h$ values. The method seems robust with respect to the change in $\rho^h$ hyperparameter for all settings. The method achieves the highest scores for $\rho^h < 0.35$ with slight differences. Reducing $\rho^h$ increases the AUC slightly for the seen and the unseen settings. For the one-class setting, the method fluctuates for some values of $\rho^h$, although not significant. Thus, using $\rho^h < 0.35$ seems advantageous due to the running time gain for all settings and a slight performance gain for the seen and the unseen settings. The scores of the ADDML for all $\rho^h$ values in Fig.~\ref{fig:rho_h_analysis} are higher than the scores of the other methods defined in Section~\ref{sec:unsupervised_experiments} for the \textit{Letter} dataset.
\section{Conclusion}
\label{sec:conclusion}

We have introduced a novel anomaly detection method using the DML. Our method projects the data into a possibly lower-dimensional latent space by minimizing the similarity loss, which aims to bring similar instances closer in the metric space, where the model is optimized "end-to-end". For the test stage, we derive a scoring function with $\Theta(1)$ time and memory complexity, i.e., the distance to the center, to score the outlierness of the instances. Our approach is generic so that it can be adapted to any type of data such as images or time series by only changing the underlying neural network architecture, for which we provide such adaptations as remarks. Moreover, our method is generic so that one can readily extend our approach using different ideas in the DML literature. Our method has demonstrated significant performance improvements compared to the state-of-the-art methods in various real-world datasets.


\ifCLASSOPTIONcaptionsoff
  \newpage
\fi

\bibliography{main.bib}
\balance
\bibliographystyle{IEEEtran}

\end{document}